\documentclass[twoside,11pt]{article}

\usepackage{jmlr2e}
\usepackage{amsmath,epstopdf}
\usepackage{float}
\usepackage[hang, small,labelfont=bf,up,textfont=it,up]{caption} % Custom captions under/above floats in tables or figures
\usepackage{subcaption}
\newtheorem{prop}[theorem]{Proposition}

\ShortHeadings{Sparse multi-class Classification}{Xia, Dong}
\firstpageno{1}

\begin{document}

\title{Exploring Sparsity in Multi-class Linear Discriminant Analysis}

\author{\name Dong Xia\thanks{Partly supported by NSF Grant DMS-1207808}\email dxia7@math.gatech.edu \\
       \addr School of Mathematics\\
       Georgia Institute of Technology\\
       Atlanta, GA 30332, USA
       }

\maketitle

\begin{abstract}%   <- trailing '%' for backward compatibility of .sty file
Recent studies in the literature have paid much attention to the sparsity in linear classification tasks. One motivation of imposing sparsity assumption
on the linear discriminant direction is to rule out the noninformative features,
making hardly contribution to the classification problem.
Most of those work were focused on the scenarios of binary classification,
such as \cite{fanfengtong2012}, \cite{cailiu2011} and \cite{maizouyuan2012}. 
In the presence of multi-class data, preceding researches recommended individually pairwise sparse linear discriminant analysis(LDA), such as \cite{cailiu2011},\cite{fanfengtong2012}.
However, further sparsity should be explored.
In this paper, an estimator of grouped LASSO type is proposed to take advantage of sparsity for multi-class data.
It enjoys appealing non-asymptotic properties which allows insignificant correlations among features.
This estimator exhibits superior capability on both simulated and real data.
\end{abstract}


\section{Introduction}
\label{intro}
Suppose that there is a collection of $i.i.d.$ random pairs $\{(X_1,Y_1),\ldots,(X_{N},Y_N)\}$. 
The vector $X_j\in\mathbb{R}^p$ contains measurements of $p$ features and the label $Y_j\in\{1,2,\ldots,K\}$ for $j=1,\ldots,N$. It is assumed that $P_{(X,Y)}(x,y)=P_Y(y)P_{X|Y}(x|y)$. The label $Y$ obeys
an unknown distribution with $\mathbb{P}(Y=j)=\pi_j,j=1,\ldots,K$ and $\sum_{j=1}^K\pi_j=1$. 
Given the sample data, the objective is to design a classifier:
\begin{equation*}
\mathcal{C}:\mathbb{R}^p\to\{1,2,\ldots,K\},
\end{equation*}
such that $\mathbb{P}_{(X,Y)}(\{Y\neq \mathcal{C}(X)\})$ is minimized. In the simplest form,
$\mathcal{C}(\cdot)$ is favored to comprise strategies based on linear functions, which is widely known as linear discriminant analysis(LDA).
The LDA model assumed that the conditional distributions $X|Y=k, k=1,\ldots,K$ are Gaussian and they are
\begin{equation*}
 X|Y=k\sim\mathcal{N}(\mu_K,\Sigma),\quad k=1,2,\ldots,K.
\end{equation*}
It is worth noting that the assumptions of Gaussian distributions can be relaxed to elliptical distributions, see \cite{cailiu2011}. 
Denote $n_k:=\sharp\{j: Y_j=k, j=1,\ldots, N\}$ for $k=1,\ldots,K$ with $\sum_{j=1}^{K}n_j=N$. LDA
performs pairwise classification via taking a linear combination of features as the criterion. More exactly, to distinguish between class $l$ and $k$ for $1\leq l\neq k\leq K$, LDA produces
the following classifier:
\begin{equation}
 \label{LDAclassifier}
 \phi_{k,l}(x):=
 \begin{cases}
  k,& \text{if } (x-(\mu_k+\mu_l)/2)'\Sigma^{-1}\delta_{k,l}+\log(\pi_k/\pi_l)>0\\
  l,& \text{otherwise}
 \end{cases}
\end{equation}
, where $\delta_{k,l}:=\mu_k-\mu_l$. It is famous that $\phi_{k,l}(x)$ is the perfect classifier which requires prerequisite knowledge of $\Sigma,\mu_l,\mu_k,\pi_l,\pi_k$.
In practice, we construct a classifier $\hat{\phi}_{k,l}(x)$ which mimics $\phi_{k,l}(x)$ by plugging corresponding estimators: $\hat{\Sigma},\hat{\mu}_{k},\hat{\mu}_l,\hat{\pi}_k,\hat{\pi}_l$ into (\ref{LDAclassifier}).
We know that in the binary case $\mathbb{P}_{(X,Y)}(\{Y\neq\hat{\phi}_{k,l}(X)\})\to\mathbb{P}_{(X,Y)}(\{Y\neq\phi_{k,l}(X)\})$ in probability when $p$ is frozen
and $\hat{\Sigma},\hat{\mu}$ are chosen as the sample covariance and sample mean respectively,
see \cite{andersonbook}. However, as proved in \cite{bickel2004some}, $\hat{\phi}_{k,l}(x)$ in this mode performs poorly in the case $p\gg N$ which now arises conventionally in various applications.
It turns out to be tricky to construct a stable estimator of $\Sigma^{-1}$ when $p\gg N$ .
Sparsity assumptions have henceforth been proposed, such as \cite{fanfengtong2012}, \cite{fanfan2008}, \cite{maizouyuan2012} and \cite{shao2011}.
There are two directions for the motivations of raising sparsity assumptions. One is that the sparsity assumption on $\Sigma$ or $\Sigma^{-1}$ enables us to propose advantageous estimators through convex optimization, 
such as \cite{yuan2010high} and \cite{cai2012estimating}.
The other direction is to impose sparsity assumptions directly on the Bayes direction $\beta_{k,l}=\Sigma^{-1}\delta_{k,l}$, see \cite{cailiu2011} and \cite{maizouyuan2012}. It corresponds to the 
situation that merely a small portion of the features is relevant to the classification problem,  which leads to a favorable interpretation.
Actually, the sparsity on $\Sigma^{-1}$ and $\delta_{k,l}$ indicates the sparsity of $\beta_{k,l}$.
In this paper, the Bayes directions: $\beta_{k,l}$ are presumed to be sparse for $1\leq k\neq l\leq K$.\\
We begin by introducing the notations and definitions. Let $\mathcal{Y}:=\{Y_1,Y_2,\ldots,Y_N\}$. Define the set
\begin{equation*}
 \mathcal{S}_{+}^p:=\{A\in\mathbb{R}^{p\times p}: A=A', A\succeq 0\}.
\end{equation*}
We denote by $T_{k,l}$ the support of $\beta_{k,l}$ and by $s_{k,l}$ the cardinality of $T_{k,l}$ for $1\leq k\neq l\leq K$.
Let $\hat{\mu}_k=n_k^{-1}\sum_{Y_j=k}X_j$ for $k=1,\ldots,K$ and 
\begin{equation*}
S=(N-K)^{-1}\sideset{}{_{k=1}^K}\sum\sideset{}{_{Y_j=k}}\sum(X_j-\hat{\mu}_k)(X_j-\hat{\mu}_k)'.
\end{equation*}
Let $\hat{\delta}_{k,l}=\hat{\mu}_k-\hat{\mu}_l$. Meanwhile, suppose that $s=\text{card}(T)\ll p$ where $T=\bigcup\limits_{k=2}^KT_{1,k}$.
To conduct pairwise discrimination, there is no need to estimate each $\beta_{k,l}$ for $1\leq k\neq l\leq K$ on account of $\beta_{k,l}=\beta_{1,l}-\beta_{1,k}$.
Consequently, it is sufficient to estimate $\beta_{1,k}$ for $k=2,\ldots,K$. For the sake of brevity, define $\beta_k:=\beta_{1,k+1}$,
$\delta_k=\delta_{1,k+1}$ for $k=1,\ldots,K'$ with $K'=K-1$. More exactly, suppose $\beta_k=(\beta_k^1,\beta_k^2,\ldots,\beta_k^p)$ for $k=1,\ldots,K'$. Given any $1\leq j\leq p$,
define $\beta^j=:(\beta_1^j,\ldots,\beta_{K'}^j)$, namely by stacking all the $j$-th entry of $\beta_k,1\leq k\leq K'$ into one vector. The vector $\beta^j$ 
is associated with the role of the $j$-th feature in the classification problem.
Define $\Delta_k=\left<\Sigma^{-1}\delta_k,\delta_k\right>$ for $1\leq k\leq K'$ and $\Delta=\sum_{k=1}^{K'}\Delta_k$. 
For any matrix $A\in\mathcal{S}_{+}^p$, we adopt the following notations: 
$A_{\min}^{+}=\min_{1\leq j\leq p}A_{jj}$, $A_{\max}^{+}=\max_{1\leq j\leq p}A_{jj}$ and $A^-_{\max}=\max_{1\leq i\neq j\leq p}|A_{ij}|$.
Let $A_{T,:}$ and $A_{:,T}$ denote the submatrix of $A$ with corresponding rows and columns. 
Denote $\delta^T$ the subvector of $\delta$ with entries indexed by $T$. Let $T^c$ denote the complement of $T$.
For any $v\in\mathbb{R}^p$, let $\|v\|$ be the usual $l_2$ norm
and $|v|_{\infty}=\max_{1\leq i\leq p}|v_i|$. We also define $(x)_+:=x\mathbb{I}(x\geq 0)$ as the truncation function where $\mathbb{I}(\cdot)$
is the indicator function.\\
The following estimator was employed for sparse LDA when $K=2$ in \cite{kolarliu2013} and \cite{fanfengtong2012}.
\begin{equation}
\label{ROADest}
 \hat{\beta}_{1}:=\underset{\hat{\delta}_{1}'\beta=1}{\arg\min}\quad \frac{1}{2}\beta'S\beta+\lambda|\beta|_1
\end{equation}
The $l_1$ norm penalty is aimed at promoting a sparse solution. A similar estimator is:
\begin{equation}
\label{est1}
 \hat{\beta}_{1}:=\underset{\beta\in\mathbb{R}^p}{\arg\min}\frac{1}{2}\beta'S\beta-\hat{\delta}_{1}'\beta+\lambda|\beta|_1.
\end{equation}
The estimator (\ref{est1}) resembles the one proposed in \cite{maizouyuan2012} which is of regression type. In contrast to
these estimators of LASSO type, another estimator(LPD) which borrowed the idea of Dantzig selector was studied in \cite{cailiu2011}:
\begin{equation}
\label{LPDestimator}
 \tilde{\beta}_{1}:=\underset{\beta\in\mathbb{R}^p}{\arg\min}\{|\beta|_1: |S\beta-\hat{\delta}_{1}|_{\infty}\leq \lambda\}
\end{equation}
If $K\geq 3$, an immediate approach is to implement the above estimators for $\beta_{1},\beta_{2},\ldots,\beta_{K'}$ separately.
Its drawback resides in the ignorance of the multi-class information. One intention of imposing sparsity assumptions on $\beta_{k}$ derives from the 
objective of expelling the noninformative features displaying weak connections with the labels. It is unexceptional to expect that most the insignificant 
features will stay valueless when discriminating class $k$ and $l$ for different pairs $(k,l)$.
There is where further sparsity might be explored.  Intuitively, we hope that $\beta^u=0$ if the $u$-th feature is a nuisance feature. 
However, the individually pairwise sparse LDA is inferior to mis-include some nuisance features due to correlation and the insufficiency of data.
Actually, our simulation result in Section~\ref{numericsec} reflects that different noisy features might be mis-selected by pairwise estimation as (\ref{LPDestimator}).
Chances of making this type of mistakes indeed can be decreased based on the same data when we take into account the grouped sparsity.
To handle the grouped sparsity, we propose the following estimator:
\begin{equation}
\label{est}
 (\hat{\beta}_1,\ldots,\hat{\beta}_{K'}):=\underset{\beta_1,\ldots,\beta_{K'}\in\mathbb{R}^p}{\arg\min}\sum\limits_{k=1}^{K'}\frac{1}{2}\beta_k'S\beta_k-\sum\limits_{k=1}^{K'}\hat{\delta}_{k}'\beta_k+\sum\limits_{j=1}^p\lambda_j||\beta^j||.
\end{equation}
The regularization parameters $\lambda_j, j=1,\ldots,p$ are positive and can be decided practically through cross-validation.
Theoretic analysis will confirm that carefully selected $\lambda_j,j=1,\ldots,p$ can yield attractive performances of (\ref{est}).
It is apparent that, when $K=2$ and $\lambda_j=\lambda,j=1,\ldots,p$,
(\ref{est}) is reduced to the commonly studied estimator (\ref{est1}).
Meanwhile, it is easy to verify the convexity of the optimization problem in (\ref{est}), which can be solved efficiently by many off-the-shelf
algorithms.
The estimator (\ref{est}) is analogous to the LASSO estimator accommodated for problems either with grouped sparsity,\cite{yuan2006model} or of multi-task regression, \cite{lounici2009taking}.
We should point out that grouped sparsity for multi-class classification has been considered in \cite{merchante2012efficient} 
in a linear regression style combined with optimal scoring. Comparable methods can be also found in \cite{zhu2014sparse} which was
used to classify Alzheimer's disease.
Variable selection for multi-class data has been studied experimentally in \cite{le2011sparse} based on partial least square discriminant analysis.
People also studied the classification task for multi-labeled data in \cite{han2010multi}, in which case each $Y_j$ may have multiple entries. 
In addition, the paper by \cite{witten2011penalized} proposed a penalized Fisher discriminant method that can be extended to the multi-class situation,
which, however, is non-convex and thereby is deficient in theoretic guarantees of its performance.
 After the completion of this paper, we noticed that \cite{maiyangzou2014} proposed the same estimator as (\ref{est}), where its theoretic properties were also studied.
 The analysis in our paper is completely different and our simulation results emphasized on the advantages of (\ref{est}) over the
 individually pairwise classification.\\
The paper will be organized as follows. In section~\ref{theorysec}, some theoretic properties of the estimator will be presented. 
Then experimental results on both simulated and real data will be reported in Section~\ref{numericsec},~\ref{realsim}, in which we will compare the performance of (\ref{est}) and (\ref{LPDestimator}), (\ref{est1}).

\section{Theoretic properties}\label{theorysec}
In this section, we turn to the theoretic properties of estimator (\ref{est}). The upper bound of the estimation error $\|\hat{\beta}^j-\beta^j\|, j=1,\ldots,p$
will be provided as long as $\frac{\Sigma^-_{\max}}{\Sigma^+_{\min}}$ is small enough.
It is well-known that $\hat{\mu}_1,\ldots,\hat{\mu}_K,S$ are mutually independent and $S=\frac{1}{N-K}ZZ'$ where $Z\in\mathbb{R}^{p\times (N-K)}$,
see \cite[Theorem 3.1.2]{Muirhead}.
Every column of $Z$ has distribution as $\mathcal{N}(0,\Sigma)$ and they are $i.i.d.$. Meanwhile, we can check that conditioned on $\mathcal{Y}$,
\begin{equation}
\label{deltadist}
 \hat{\delta}_k\sim\mathcal{N}(\delta_k,\frac{n_1+n_{k+1}}{n_1n_{k+1}}\Sigma),\quad k=1,\ldots,K'.
\end{equation}
Lemma~\ref{ncon} uncovers the concentration of $n_i, 1\leq i\leq K$, which will be useful in the proof of our main theorem. Similar inequalities as
in Proposition~\ref{prop1} appear regularly in researches of compressed sensing and low rank matrix completion, see \cite{koltchinskii2011oracle}.
\begin{lemma}
\label{ncon}
 There exists an event $\mathcal{A}$ with $\mathbb{P}(\mathcal{A})\geq 1-2\sum\limits_{k=1}^{K}e^{-N\pi_k/16}$ such that on $\mathcal{A}$,
 \begin{equation*}
  n_k\in\left[\frac{N\pi_k}{2},\frac{3N\pi_k}{2}\right],\quad 1\leq k\leq K.
 \end{equation*}
\end{lemma}
\begin{prop}
\label{prop1}
 Let $\bar{\pi}=\max_{2\leq k\leq K}\sqrt{\frac{\pi_1+\pi_{k}}{\pi_1\pi_{k}}}$ and $(\hat{\beta}_1,\ldots,\hat{\beta}_{K'})$ be the
 solution of (\ref{est}). Then for any $t>0$, there exists an event $\mathcal{B}_t$ with $\mathbb{P}(\mathcal{B}_t)\geq\mathbb{P}(\mathcal{A})-Kpe^{-t}-\frac{2Kp}{\sqrt{\pi}t}e^{-t^2}$
 such that on $\mathcal{B}_t$ we have
 \begin{equation*}
  \sum\limits_{j=1}^p\lambda_j\|\hat{\beta}^j\|\leq \sum\limits_{j=1}^p\lambda_j\|\beta^j\|+C_0\sum\limits_{j=1}^p\|\hat{\beta}^j-\beta^j\|\sqrt{\frac{\bar{\pi}\Sigma_{jj}(\Delta\vee K)t}{N-K}},
 \end{equation*}
 where $C_0>0$ is a universal constant. 
 Furthermore, if $\lambda_j=\lambda,j=1,\ldots,p$ with 
 \begin{equation}
 \label{prop1lambda}
  \lambda= 2C_0\sqrt{\frac{\bar{\pi}\Sigma_{\max}^{+}(\Delta\vee K)t}{N-K}},
 \end{equation}
 then on event $\mathcal{B}_t$, we have
 \begin{equation*}
  \sum\limits_{j\notin T}\|\hat{\beta}^j\|\leq 3\sum\limits_{j\in T}\|\hat{\beta}^j-\beta^j\|.
 \end{equation*}
\end{prop}
Moreover, if $\lambda_j=2C_0\sqrt{\frac{\bar{\pi}\Sigma_{jj}(\Delta\vee K)t}{N-K}},j=1,\ldots,p$ in Proposition~\ref{prop1}, it leads to $ \sum\limits_{j\notin T}\lambda_j\|\hat{\beta}^j\|\leq 3\sum\limits_{j\in T}\lambda_j\|\hat{\beta}^j-\beta^j\|$
 on event $\mathcal{B}_t$.
Let $\mathcal{D}$ denote the event: $\mathcal{D}:=\big\{\{S^-_{\max}\leq 2\Sigma^-_{\max}\}\cap \{S^+_{\min}\geq \frac{1}{2}\Sigma^+_{\min}\}\big\}$.
\begin{prop}
 \label{prop2}
 Suppose that $\lambda_j=\lambda, j=1,\ldots,p$ with $\lambda$ chosen in (\ref{prop1lambda}).
 On the event $\mathcal{D}\cap\mathcal{B}_t$ for any $t>0$, we have
 \begin{equation}
  \label{prop2ineq1}
  \Big(\frac{\Sigma^+_{\min}}{2}-32s\Sigma^-_{\max}\Big)\sum\limits_{j\in T}\|\hat{\beta}^j-\beta^j\|^2+\frac{\Sigma^+_{\min}}{2}\sum\limits_{j\notin T}\|\hat{\beta}^j\|^2\leq 12\lambda\sum\limits_{j\in T}\|\hat{\beta}^j-\beta^j\|.
 \end{equation}
Furthermore, if $\Sigma^+_{\min}\geq 128s\Sigma^-_{\max}$, we have on the event $\mathcal{D}\cap\mathcal{B}_t$,
\begin{equation}
\label{prop2ineq2}
 \sum\limits_{j\in T}\|\hat{\beta}^j-\beta^j\|^2\leq \frac{C_1\lambda^2s}{(\Sigma^+_{\min})^2}
\end{equation}
and 
\begin{equation}
\label{prop2ineq3}
 \sum\limits_{j\notin T}\|\hat{\beta}^j\|^2\leq \frac{C_1\lambda^2s}{2(\Sigma^+_{\min})^2}.
\end{equation}
for some constant $C_1>0$.
\end{prop}

\begin{prop}
 \label{prop3}
 Let $\lambda_j=\lambda,j=1,\ldots,p$ with $\lambda$ chosen as $(\ref{prop1lambda})$. Meanwhile, suppose that $\Sigma^+_{\min}\geq 128s\Sigma^-_{\max}$.
 For any $t>0$, on the event $\mathcal{D}\cap\mathcal{B}_t$, we have
 \begin{equation*}
  \sup\limits_{1\leq j\leq p}\|\hat{\beta}^j-\beta^j\|\leq \frac{C_2\lambda}{\Sigma^+_{\min}}=2C_0C_2\sqrt{\frac{\bar{\pi}\Sigma^+_{\max}(\Delta\vee K)t}{(\Sigma^+_{\min})^2(N-K)}},
 \end{equation*}
 where $C_2=4+\frac{\sqrt{C_1}}{8}$.
\end{prop}

For any $\zeta>0$, we define a thresholding function $\phi_{\zeta}(\cdot):\mathbb{R}\to\mathbb{R}$ as
\begin{equation*}
\phi_{\zeta}(x)=
 \begin{cases}
  x,& |x|\geq \zeta\\
  0,& |x|<\zeta.
 \end{cases}
\end{equation*}
When we apply the function $\phi_{\zeta}$ to a vector $\beta$, it means we apply $\phi_{\zeta}$ to each entry of $\beta$. Theorem~\ref{mainthm1}
follows immediately from Proposition~\ref{prop3}, Lemma~\ref{eventDlem} and the definition of $\phi_{\zeta}(\cdot)$, which provides a sufficient condition for the
support recovery of our estimator. In the case that $\Sigma^+_{\min}\approx\Sigma^+_{\max}$, the lower bound on $\min_{j\in T_k}|\beta_k^j|$
is of the order $O(\sqrt{\frac{(\Delta\vee K)\log(p\vee N)}{N}})$, which is similar to the necessary lower bound on the non-trivial entries of $\beta$
for sign consistency of (\ref{ROADest}) when $K=2$, see \cite{kolarliu2013}.
\begin{theorem}
 \label{mainthm1}
 Under the same conditions of Proposition~\ref{prop3} and suppose that there exists some constants $C_3>0$ which
 are large enough such that $N\geq K+C_3\big(\frac{\Sigma^+_{\max}}{\Sigma^+_{\min}}\big)^2s^2\log(p\vee N)$. Then we have, with probability at
 least $1-2\sum\limits_{j=1}^{K}e^{-N\pi_K/16}-\frac{2}{p\vee N}$
 \begin{equation*}
  |\hat{\beta}_k-\beta_k|_{\infty}\leq 4C_0C_2\sqrt{\frac{\bar{\pi}\Sigma^+_{\max}(\Delta\vee K)\log(p\vee N)}{(\Sigma^+_{\min})^2(N-K)}}=:\zeta.
 \end{equation*}
Furthermore, suppose that for any $1\leq k\leq K'$, $\min_{j\in T_k}|\beta_k^j|> 2\zeta$. 
Define $\hat{\hat{\beta}}_k=\phi_{\zeta}(\hat{\beta}_k)$ for $1\leq k\leq K'$, then we have
\begin{equation*}
 \textrm{Supp}(\hat{\hat{\beta}}_k)=\textrm{Supp}(\beta_k),\quad 1\leq k\leq K',
\end{equation*}
with the same probability.
\end{theorem}
In Theorem~\ref{mainthm2}, a lower bound on the estimation error of $\beta$ is given by assuming that $T_1,\ldots,T_{K'}$ are known in advance.
Under the circumstances, the ideal estimators would be $\bar{\beta}_k^{T_k}=S_{T_k,T_k}^{-1}\hat{\delta}_k$ and $\bar{\beta}_k^{T_k^c}=0$ for $1\leq k\leq K'$.
We then calculate $\mathbb{E}\|\bar{\beta}^j-\beta^j\|,1\leq j\leq p$ which can be regarded as benchmarks for the estimation errors.
It confirms the optimality(except the logarithmic term) of the bound in Proposition~\ref{prop3} for $j\in\bigcap_{k=1}^{K'} T_k$ when $\Sigma^+_{\min}\approx\Sigma^+_{\max}$.
It should be noted that under the conditions of Proposition~\ref{prop3}, we have $\Sigma^+_{\min}\|\beta_k\|^2\leq 2\Delta_k$ for $1\leq k\leq K'$.
\begin{theorem}
\label{mainthm2}
Suppose we have access to $T_1,\ldots,T_{K'}$ and $\bar{\beta}_k$ are defined as above for $1\leq k\leq K'$. Let $\underline{\pi}=\min_{2\leq k\leq K}\sqrt{\frac{\pi_1+\pi_{k}}{\pi_1\pi_{k}}}$.
For any $1\leq j\leq p$, define $\kappa_j=\min_{1\leq k\leq K}(\Sigma_{T_k,T_k}^{-1})_{jj}$, $\bar{\Delta}_j:=\sum_{k=1,j\in T_k}^{K'}\Delta_k$,
$K_j:=\sharp\{1\leq k\leq K': j\in T_k\}$ and $\omega_j=\sum_{k=1,j\in T_k}^{K'}|\beta_k^j|^2$, then
\begin{equation*}
 \mathbb{E}\|\bar{\beta}^j-\beta^j\|_2^2\geq \frac{1}{2} \Big[\frac{\bar{\Delta}_j\kappa_j}{N-K}+\frac{K_j\kappa_j\underline{\pi}}{N-K}+\frac{\omega_j}{N-K}\Big].
\end{equation*}
\end{theorem}

\section{Algorithm to solve (\ref{est})}
In this section, we briefly discuss how to adapt one existing algorithm to solve the minimization problem in (\ref{est}).
Let 
$$
f(\beta_1,\ldots,\beta_{K'}):=\frac{1}{2}\sum\limits_{k=1}^{K'}\beta_k'S\beta_k-\sum\limits_{k=1}^{K'}\hat{\delta}_{k}'\beta_k.
$$
We will utilize the scheme in \cite{liu2010efficient}. The method attempts to approximate $f(\beta_1,\ldots,\beta_{K'})$ by
\begin{equation*}
 Af_{(\tilde{\beta}_1,\ldots,\tilde{\beta}_{K'})}(\beta_1,\ldots,\beta_K)=f(\tilde{\beta}_1,\ldots,\tilde{\beta}_{K'})+\left<\tilde{\nabla}_f,\beta-\tilde{\beta}\right>+\frac{L}{2}||\beta-\tilde{\beta}||^2.
\end{equation*}
The parameter $L>0$ controls the deviation of $\beta$ from $\tilde{\beta}$ and $\tilde{\nabla}_f$ denotes the gradient of $f$ at $(\tilde{\beta}_1,\ldots,\tilde{\beta}_{K'})$.
Then the accelerated gradient algorithm is applied to the function $Af_{(\tilde{\beta}_1,\ldots,\tilde{\beta}_{K'})}(\beta_1,\ldots,\beta_K)+\sum\limits_{j=1}^p\lambda_j\|\beta^j\|$. 
It updates $(\tilde{\beta}_1,\ldots,\tilde{\beta}_{K'})$ and $(\hat{\beta}_1,\ldots,\hat{\beta}_{K'})$ alternatively. One of the 
key points of this algorithm is that the solution of the following optimization problem,
\begin{equation*}
 \hat{v}:=\underset{v\in\mathbb{R}^K}{\arg\min}\frac{1}{2}||v-x||^2+\lambda||v||
\end{equation*}
has a closed form as $\hat{v}:=\frac{(\|x\|-\lambda)_{+}}{\|x\|}x$. This algorithm inherits the $O(1/k^2)$ convergence rate of the accelerated 
gradient method.

\section{Numerical Simulations}\label{numericsec}
In this section, we will compare the performance of (\ref{est}) and (\ref{LPDestimator}) on simulated data. As stated in \cite{cailiu2011},
(\ref{LPDestimator}) can be formulated into a linear programming(LP) problem. The built-in LP solver in MATLAB works efficiently when $p$ is not tremendous. 
Actually, we set $n_1=n_2=n_3=20$ and $p=200$. The purpose of this simulation is to demonstrate the power
of our estimator in variable selection. The result reveals that by implementing sparse LDA individually from (\ref{LPDestimator}),
some nuisance features are mis-selected into the model. This can be prevented by our estimator (\ref{est}).
It should be noted that $\Sigma$ and $\mu_1,\mu_2,\mu_3$ are chosen quite generally without much special design in the simulation. 
Let $\Sigma\in\mathcal{S}^{p}_{+}$ be
\begin{equation*}
 \Sigma(4,1:3)=(1/4,1/3,1/4),\quad \Sigma(5,1:3)=(1/5,-1/4,1/5),\quad \textrm{diag}(\Sigma)=1.
\end{equation*}
Then we set $\beta_{1}=(-2,3,1,0,\ldots,0)^{\top}\in\mathbb{R}^p$, $\beta_2=(1,-2,-1.2,0,\ldots,0)\in\mathbb{R}^p$ and $\mu_1=0$.
The vectors $\mu_2$ and $\mu_3$ are determined in line with the facts that $\mu_2=\Sigma\beta_1$ and $\mu_3=\Sigma\beta_2$. Denote $\tilde{\beta}_1$
and $\tilde{\beta}_2$ the solutions obtained from the LPD estimators (\ref{LPDestimator}) independently.
Figure~\ref{sim1_indep} shows how the entries of $\tilde{\beta}_1$ and $\tilde{\beta}_2$ vary accordingly as $\lambda$ grows.
The variable selection process of $\tilde{\beta}_1$ and $\tilde{\beta}_2$ indicates the weakness in estimating $\beta_1$ and $\beta_2$
separately, owing to the scarcity of data. Indeed, incorporating inessential features occurs frequently when $N$ is small enough compared with $p$.\\
\begin{figure}[h]
\centering
\begin{subfigure}{.5\textwidth}
  \centering
  \includegraphics[height=3.0in,width=3.2in]{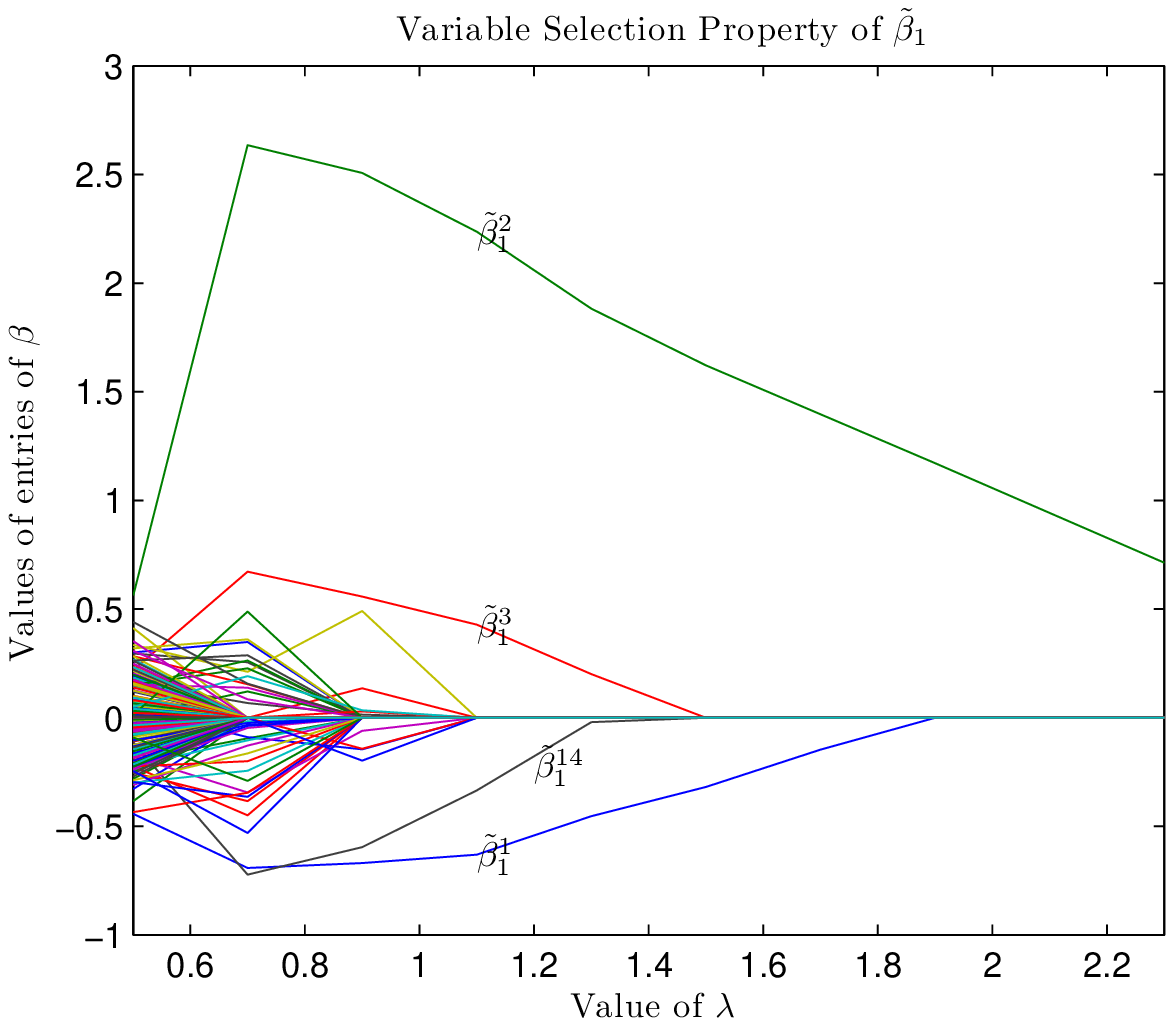}
  \caption{Variable Selection property of $\tilde{\beta}_1$}
  \label{sim1_indep_sub1}
\end{subfigure}%
\begin{subfigure}{.5\textwidth}
  \centering
  \includegraphics[height=3.0in,width=3.2in]{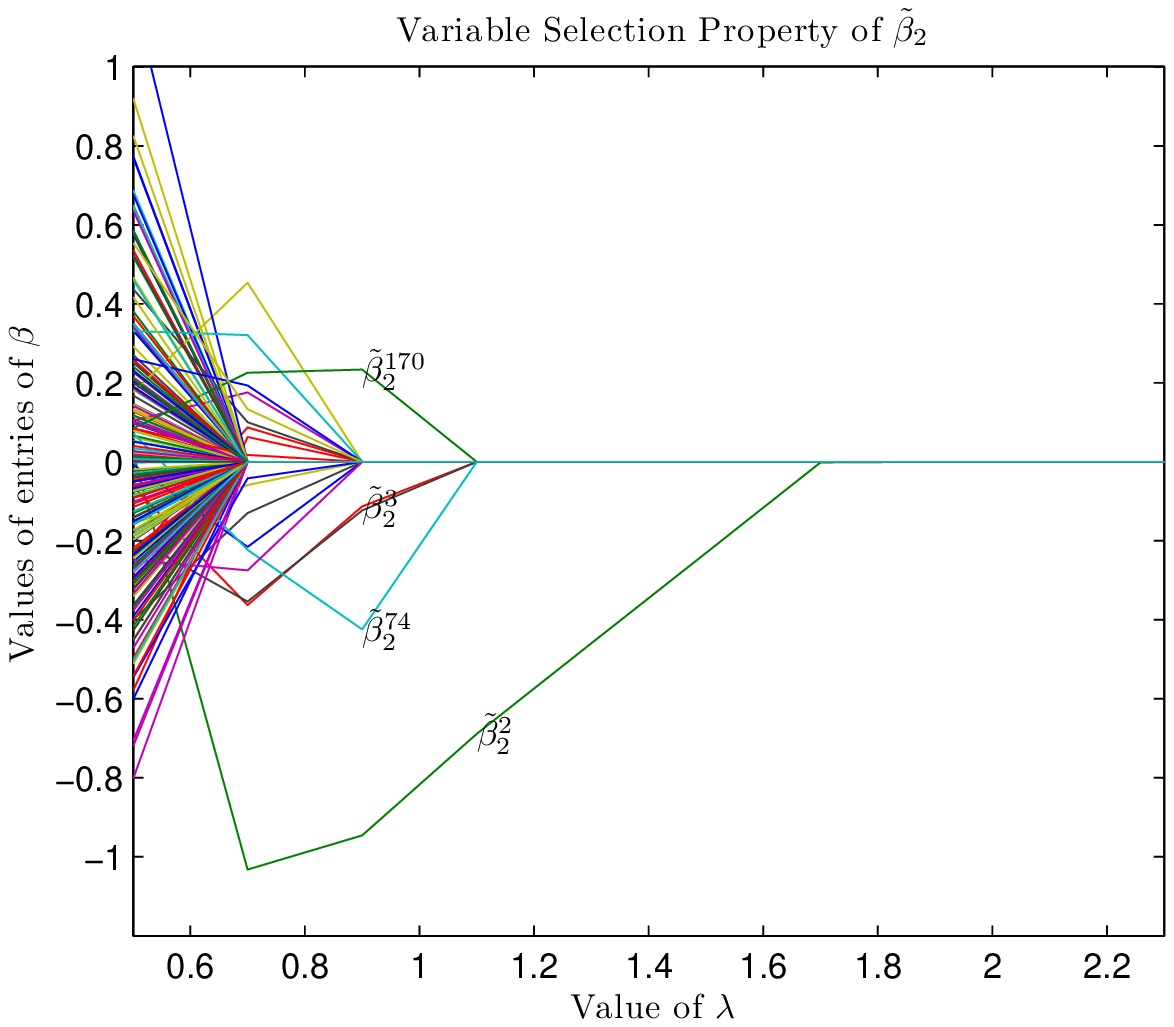}
  \caption{Variable Selection property of $\tilde{\beta}_2$}
  \label{sim1_indep_sub2}
\end{subfigure}
\caption{Simulation $1$: Variable selection properties of $\tilde{\beta}_1$ and $\tilde{\beta}_2$ which are estimated individually via the LPD estimator
when $\lambda$ is increasing. The non-vanished entries of $\tilde{\beta}_1$ and $\tilde{\beta}_2$ are depicted in the plots. The sample size $n_1=n_2=n_3=20$,
which is much smaller than the number of features $p=200$. The result apparently attests that by estimating the Bayes' directions separately,
it is likely to mis-include the nuisance features into our model.}
\label{sim1_indep}
\end{figure}
Then we switch to apply group sparsity in estimating $\beta_1$ and $\beta_2$ together. By choosing $\lambda_1=\ldots=\lambda_p=\lambda$,
our estimator works as follows
\begin{equation*}
 (\hat{\beta}_1,\hat{\beta}_2):=\underset{(\beta_1,\beta_2)\in(\mathbb{R}^p,\mathbb{R}^p)}{\arg\min}\frac{1}{2}
 \sum\limits_{k=1}^2\beta_k'S\beta_k-\sum\limits_{k=1}^2\hat{\delta}_k'\beta_k+\lambda\sum\limits_{j=1}^p||\beta^{j}||.
\end{equation*}
The variable selection property of $(\hat{\beta}_1,\hat{\beta}_2)$ is also examined as given in Figure~\ref{sim1_group}. Compared with
Figure~\ref{sim1_indep}, it is evident that the grouped sparse LDA works better in feature selection.\\
\begin{figure}[h]
\centering
\begin{subfigure}{.5\textwidth}
  \centering
  \includegraphics[height=3.0in,width=3.2in]{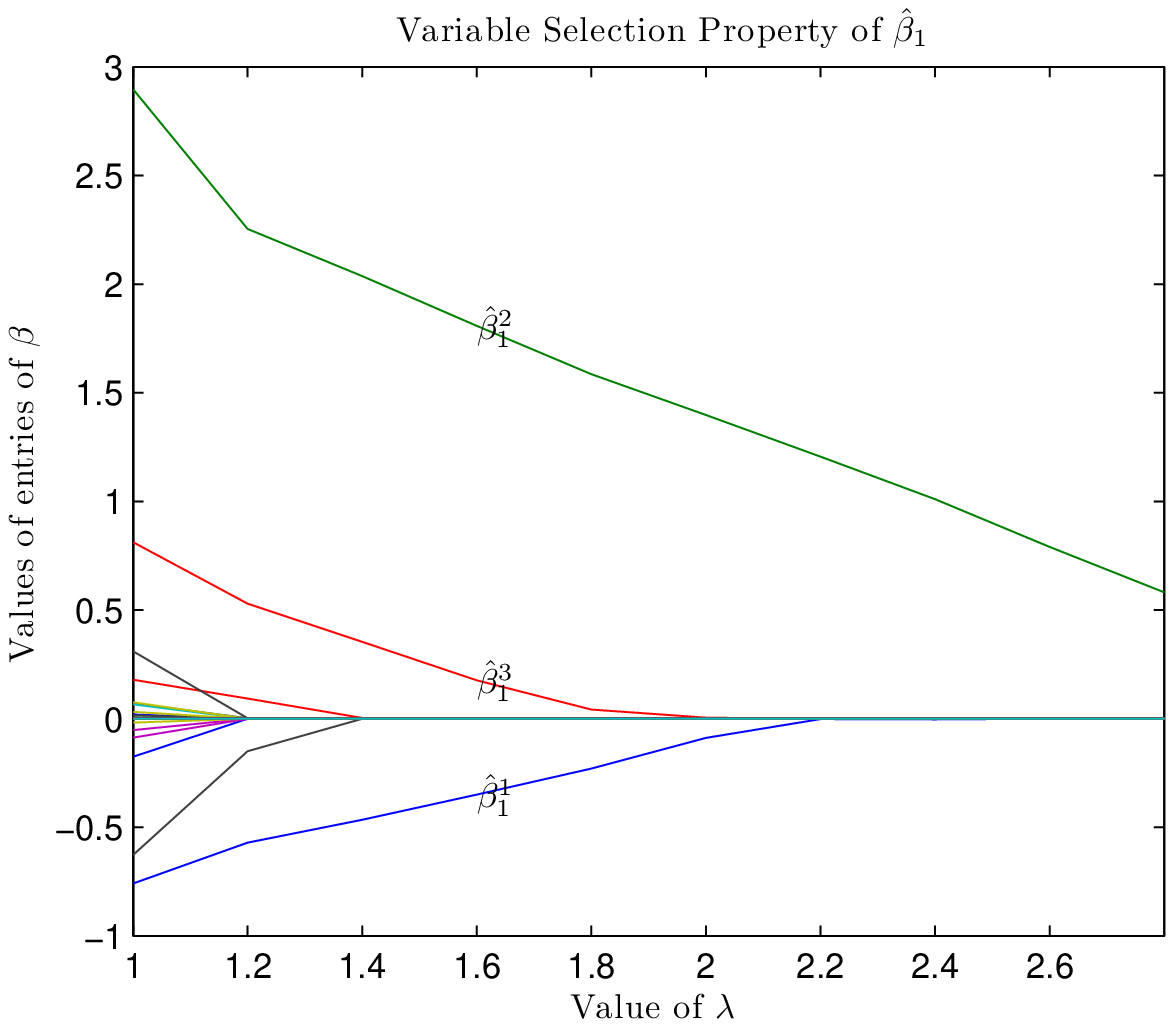}
  \caption{Variable Selection property of $\hat{\beta}_1$}
  \label{sim1_group_sub1}
\end{subfigure}%
\begin{subfigure}{.5\textwidth}
  \centering
  \includegraphics[height=3.0in,width=3.2in]{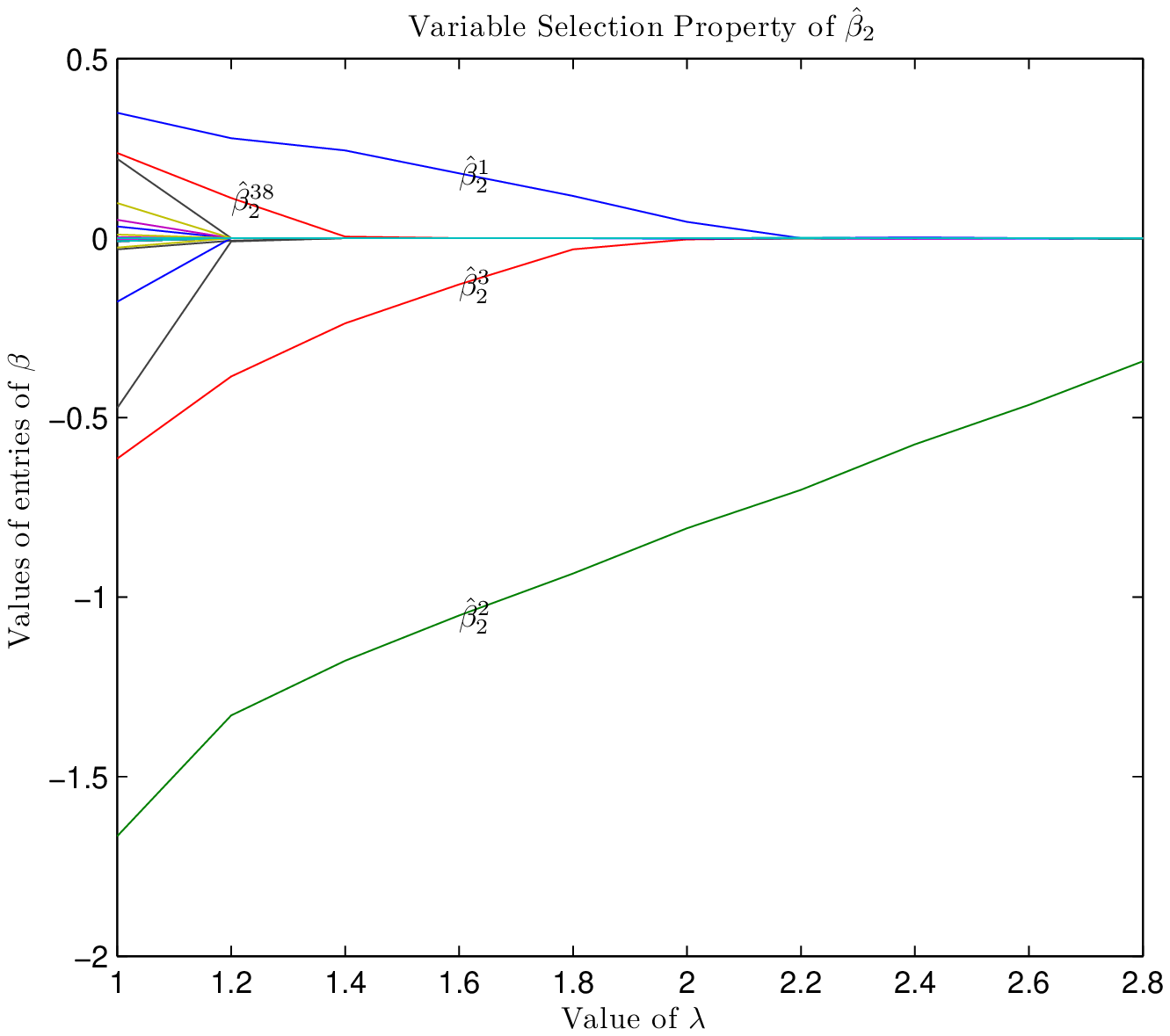}
  \caption{Variable Selection property of $\hat{\beta}_2$}
  \label{sim1_group_sub2}
\end{subfigure}
\caption{Simulation $1$: Variable selection properties of $(\hat{\beta}_1,\hat{\beta}_2)$ which are estimated through exploiting grouped sparsity
with $\lambda$ growing. It confirms that both $\hat{\beta}_1$ and $\hat{\beta}_2$ are able to filter out the negligible features when $\lambda$ is 
attentively chosen.}
\label{sim1_group}
\end{figure}
In our second simulation, we consider more complex $\Sigma$ and $T_1\neq T_2$. In fact, $\Sigma$ is chosen as:
\begin{equation*}
 \Sigma_{i,j}=\frac{1}{3^{|i-j|}},\quad\text{for } 1\leq i, j\leq p/2\quad\textrm{and}\quad\textrm{diag}(\Sigma)=1
\end{equation*}
Therefore, the correlations exist exclusively within the first $p/2$ features and the remaining ones are pure noise. 
Let $\beta_1=(-1.5,1,0,2,0,\ldots,0)^{\top}\in\mathbb{R}^p$, $\beta_2=(1,-1.8,-2,0,\ldots,0)^{\top}\in\mathbb{R}^p$ and $\mu_1=0$.
Clearly, $\beta_1$ and $\beta_2$ have different supports. We sampled $n_1=n_2=n_3=20$ data points.  Since $T_1\neq T_2$, it is not likely that the grouped sparsity estimator
can identify the features correctly for both $\hat{\beta}_1$ and $\hat{\beta}_2$. For this reason, there is no obvious evidence of advantages 
for either (\ref{est}) or (\ref{LPDestimator}) in the same sense of feature selection as in the
former simulation.
Instead, we inspect the $l_2$ magnitude of Bayes directions for each feature, $i.e.$, the values of $||\tilde{\beta}^j||$ and $||\hat{\beta}^j||$
for $j=1,\ldots,p$. The outcome is presented in Figure~\ref{sim2}.
\begin{figure}[h]
\centering
\begin{subfigure}{.5\textwidth}
  \centering
  \includegraphics[height=3.0in,width=3.1in]{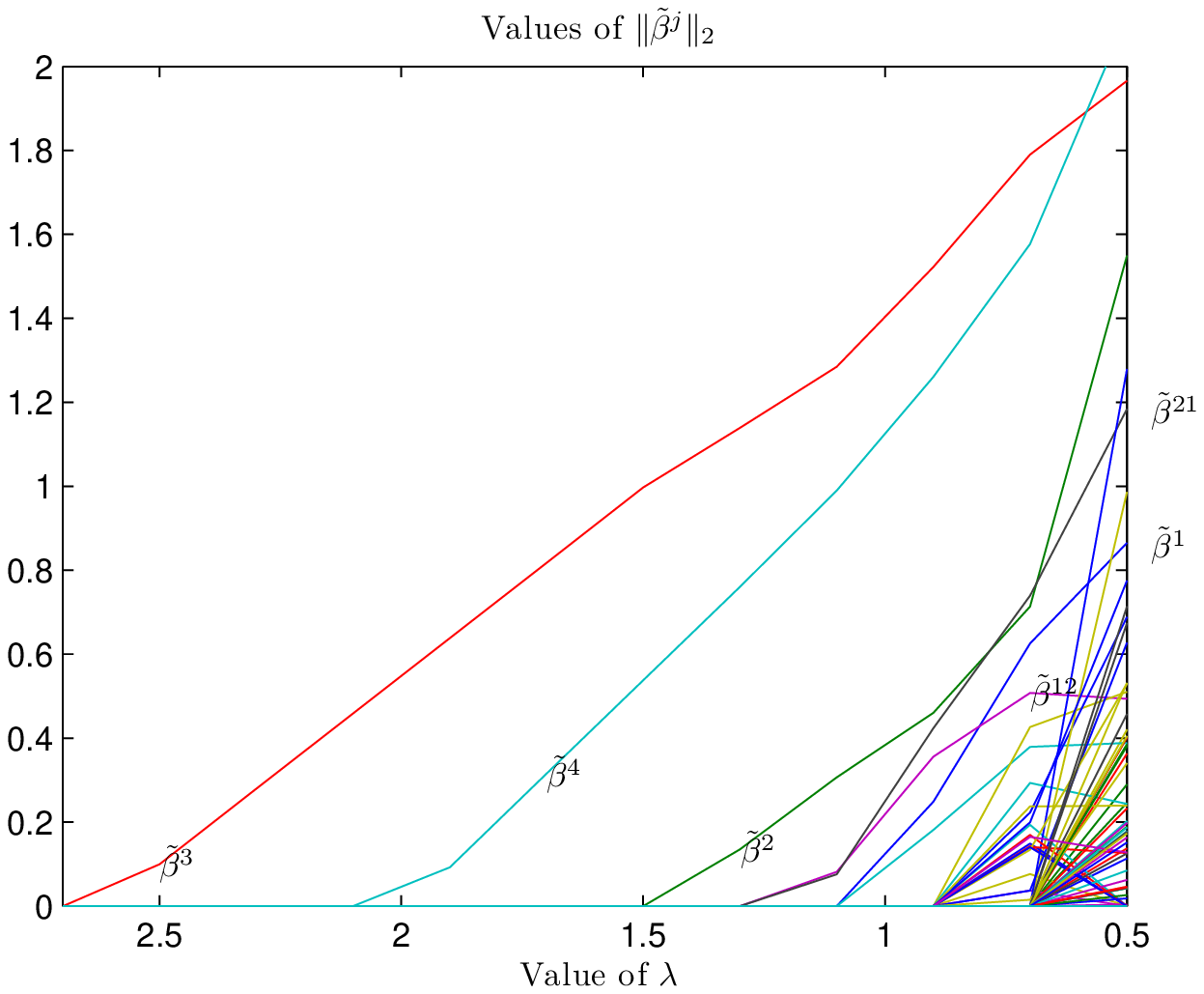}
  \caption{Values of $\|\tilde{\beta}^j\|_2$}
  \label{sim2_single}
\end{subfigure}%
\begin{subfigure}{.5\textwidth}
  \centering
  \includegraphics[height=3.0in,width=3.1in]{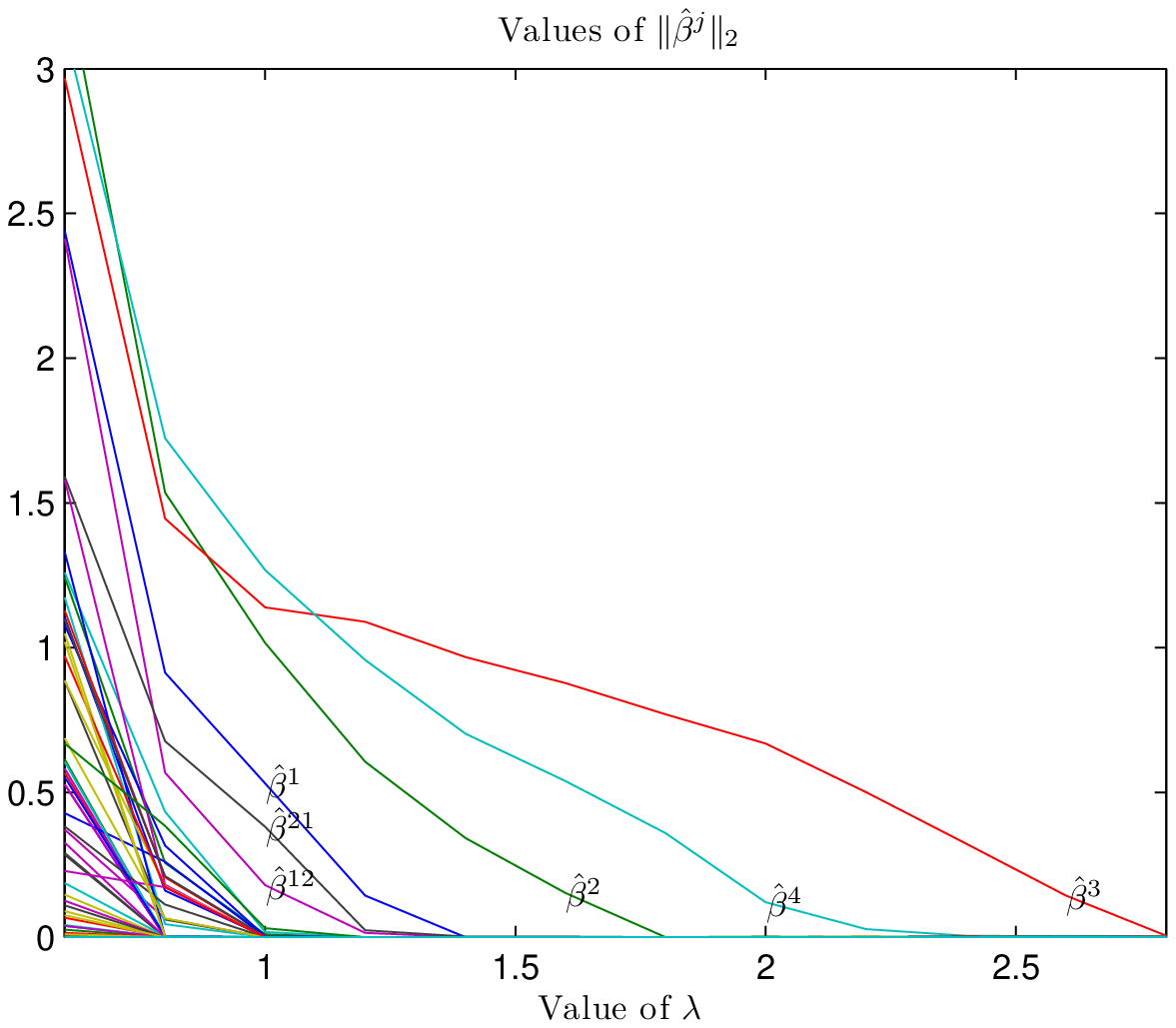}
  \caption{Values of $\|\hat{\beta}^j\|_2$}
  \label{sim2_group}
\end{subfigure}
\caption{Simulation $2$: values of $\|\hat{\beta}^j\|_2$ and $\|\tilde{\beta}^j\|_2$ for $j=1,\ldots,p$. In sub-Figure~\ref{sim2_single},
it shows that the estimators $(\tilde{\beta}_1,\tilde{\beta}_2)$ places the irrelative features $12$ and $21$ in front of the relative feature $1$
as being more important. On the contrary, the grouped sparsity estimator $(\hat{\beta}_1,\hat{\beta}_2)$ produces the desired performance. The result in sub-Figure~\ref{sim2_group} entails that the four most
informative features are actually $T_1\cup T_2$.}
\label{sim2}
\end{figure}

\section{Experiments on real datasets}\label{realsim}
In this section, we will implement our estimators on several datasets. Our experiment is conducted on three datasets: GLIOMA dataset,
MLL dataset, SRBCT dataset. These pre-processed datasets are available from \cite{yang2006stable}. In the GLIOMA dataset,
there are $4434$ genes features chosen from $12625$ features with largest absolute values of $t$-statistics. The dataset contains $50$ samples in four classes with $n_1=14, n_2=7, n_3=14, n_4=15$.
We split the data into a training set and a testing set. The training set contains $11,5,11,12$ samples from the four classes respectively.
The remaining samples are treated as testing data. The MLL dataset includes $72$ samples from three classes with $5848$ features. The authors of
\cite{yang2006stable} already split the datasets into a training set and testing set. Therefore, we directly adopt our estimator to the training data.
In the training set, $n_1=20, n_2=17, n_3=20$. In the testing set, it provides $4,3,8$ samples for each classes. In the SRBCT dataset, 
there are $83$ samples from four classes. The number of gene features is $2308$. The number of samples for each class is $n_1=29, n_2=11, n_3=18, n_4=25$.
We also split the dataset into a training set and a testing set. In the training set, there are $26,9,16,22$ samples for each class.\\
To run LDA, $(\hat{\beta}_1,\hat{\beta}_2,\hat{\beta}_3)$ or $(\hat{\beta}_1,\hat{\beta}_2)$ will be estimated from the training set and be employed 
to predict the labels of the testing data. Define $\hat{\pi}_j=\frac{n_j}{N}$ as the estimator of $\pi_j, j=1,2,3$ where $n_j$ is based on the training set.
The performance is certainly measured by the predicting error rate on the testing set. To demonstrate the efficiency of
our estimator, we will compare our grouped LASSO estimator to the estimator (\ref{est1}).
The regularization parameter $\lambda$ is chosen by $5$-folded cross validation on the training data. In the case that the error rates happen to be equal for different
$\lambda$, we choose the largest one. The exactly same approaches will be applied to (\ref{est}) and (\ref{est1}).
The misclassification error rates are reported in Table~\ref{realexptable}, which shows that (\ref{est}) and (\ref{est1}) have matching performance on
SRBCT and MLL datasets. However, our grouped LASSO estimator (\ref{est}) outperforms all the other estimators on GLIOMA dataset.
\begin{table}[ht]
\centering 
\begin{tabular}{c c c c} 
\hline\hline 
Estimator & GLIOMA & SRBCT & MLL \\ [0.5ex] 
$\hat{\beta}$ & 0.00 & 0.00 & 0.00\\
&(0.15)&(0.00)&(0.00)\\
$\tilde{\beta}$ & 0.18 & 0.00 & 0.00 \\
&(0.12)&(0.00)&(0.00)\\
Naive Bayes&0.64&0.90&0.53\\
$\tilde{\beta}:=S^{+}\hat{\mu}$&0.09&0.60&0.07\\
\hline 
\end{tabular}
\caption{Experiments on Real Datasets: $\hat{\beta}$ denotes our estimator of grouped LASSO type and $\tilde{\beta}$ denotes the estimator (\ref{est1}).
The regularization parameter $\lambda$ is concluded by $5$-folded cross validation. After selecting the best value of $\lambda$, 
the error rates are reported based on the testing data and the standard variances are also reported in the parentheses, measured in $5$-folded cross validation on training data. 
In comparison, the performance of Naive Bayes classifier is also given for all the three datasets. Another trivial estimator is constructed
via the pseudo-inverse of $S$.
} 
\label{realexptable}
\end{table}

\section{Proofs}
We begin by stating and proving two preliminary lemmas. Lemma~\ref{devlem1} is related to the concentration of $S\beta_k-\hat{\delta}_k$
for $1\leq k\leq K'$, while Lemma~\ref{eventDlem} will show that event $\mathcal{D}$ holds with high probability.
The vectors $e_1,\ldots,e_p$ represent the standard basis vectors in $\mathbb{R}^p$.
\begin{lemma}
\label{devlem1}
 For $1\leq k\leq K'$, $1\leq j\leq p$ and any $t>0$, then conditioned on $\mathcal{A}$, we have with probability at least $1-\frac{2}{\sqrt{\pi}t}e^{-t^2}-e^{-t}$,
 \begin{equation*}
  \big|e_j'S\beta_k-e_j'\hat{\delta}_k\big|\leq C_0\Big[\sqrt{\frac{\Sigma_{jj}\Delta_k t}{N-K}}\bigvee \frac{\sqrt{\Sigma_{jj}\Delta_k}t}{N-K}\bigvee \sqrt{\frac{24(\pi_1+\pi_{k+1})\Sigma_{jj}t}{N\pi_1\pi_{k+1}}}\Big]
 \end{equation*}
\end{lemma}
\begin{proof}
We know that $|e_j'S\beta_k-e_j'\hat{\delta}_k|\leq |e_j'S\beta_k-e_j'\Sigma\beta_k|+|e_j'\delta_k-e_j'\hat{\delta}_k|$,
thanks to the fact that $\delta_k=\Sigma\beta_k$. Based on (\ref{deltadist}), we get,
\begin{equation*}
 e_j'\hat{\delta}_k|\mathcal{Y}\sim\mathcal{N}\Big(e_j'\delta_k,\frac{(n_1+n_{k+1})\Sigma_{jj}}{n_1n_{k+1}}\Big).
\end{equation*}
By the concentration of Gaussian random variable, we have for any $t>0$,
\begin{equation*}
 \mathcal{P}\Big(|e_j'\hat{\delta}_k-e_j'\delta_k|\geq \sqrt{\frac{24(\pi_1+\pi_{k+1})\Sigma_{jj}t}{N\pi_1\pi_{k+1}}}\Big|\mathcal{A}\Big)\leq \frac{2}{\sqrt{\pi}t}e^{-t^2},
\end{equation*}
where the event $\mathcal{A}$ is defined in Lemma~\ref{ncon}. Now we seek to bound $|e_j'S\beta_k-e_j'\Sigma\beta_k|$. 
Define
\begin{equation*}
U=e_j'S\beta_k-e_j'\Sigma\beta_k=\frac{1}{N-K}\sum\limits_{i=1}^{N-K}(e_j'Z_{:,i}Z_{:,i}'\beta_k-e_j'\Sigma\beta_k)=\frac{1}{N-K}\sum\limits_{i=1}^{N-K}U_i,
\end{equation*}
where $U_i=e_j'Z_{:,i}Z_{:,i}'\beta_k-e_j'\Sigma\beta_k$ for $1\leq i\leq N-K$. Then we see that $\mathbb{E}U_i=0$ and $U_1,\ldots,U_{N-K}$
are $i.i.d.$ sub-exponential random variable. Meanwhile,
\begin{equation*}
 \mathbb{E}U_i^2\leq \mathbb{E}(e_j'Z_{:,i}Z_{:,i}'\beta_k)^2\leq\sqrt{\mathbb{E}(e_j'Z_{:,i})^4\mathbb{E}(Z_{:,i}'\beta_k)^4}\leq C_0'\Sigma_{jj}\Delta_k
\end{equation*}
By Bernstein inequality for the sum of independent sub-exponential random variable, such as \cite[Corollary 5.17]{vershynin2010introduction} we get
\begin{equation*}
 \mathbb{P}\Big(U\geq C_0 \sqrt{\Sigma_{jj}\Delta_k}\big[\sqrt{\frac{t}{N-K}}\vee \frac{t}{N-K}\big]\Big)\leq e^{-t}
\end{equation*}
for some constant $C_0>0$.
\end{proof}

\begin{lemma}
 \label{eventDlem}
 Suppose that
 $$
 N\geq K+C_3\Big[\Big(\frac{\Sigma^+_{\max}}{\Sigma^-_{\max}}\Big)^2\vee\Big(\frac{\Sigma^+_{\max}}{\Sigma^+_{\min}}\Big)^2\Big]\log(p\vee N)
 $$
 for some constant $C_3>0$, there exists an event $\mathcal{D}$ with $\mathbb{P}(\mathcal{D})\geq 1-\frac{1}{p\vee N}$ such that on $\mathcal{D}$,
 \begin{equation*}
  S^+_{\min}\geq \frac{\Sigma^+_{\min}}{2}\quad\textrm{and}\quad S^-_{\max}\leq 2\Sigma^-_{\max}
 \end{equation*}
\end{lemma}
\begin{proof}
For any $1\leq j\leq p$, consider $S_{jj}-\Sigma_{jj}=\frac{1}{N-K}\sum_{i=1}^{N-K}\xi_i^2-\Sigma_{jj}$ with $\xi_i\sim\mathcal{N}(0,\Sigma_{jj})$
being independent for $1\leq i\leq N-K$. Akin to the proof of Lemma~\ref{devlem1}, we have with probability at least $1-e^{-t}$ for any $t>0$,
\begin{equation*}
 |S_{jj}-\Sigma_{jj}|\leq C_3'\Sigma_{jj}\Big[\sqrt{\frac{t}{N-K}}\bigvee \frac{t}{N-K}\Big].
\end{equation*}
Similarly we can get for $1\leq i\neq j\leq p$,
\begin{equation*}
 \mathbb{P}\Big(\big|S_{ij}-\Sigma_{ij}\big|\geq C_3'\sqrt{\Sigma_{ii}\Sigma_{jj}}\Big[\sqrt{\frac{t}{N-K}}\bigvee \frac{t}{N-K}\Big]\Big)\leq e^{-t}.
\end{equation*}
The proof is completed after $t$ is adjusted to be $c_3'\log(p\vee N)$ for some constant $c_3'>0$.
\end{proof}

\begin{proof}{\bf of Lemma~\ref{ncon}}
Let $\epsilon_1,\ldots,\epsilon_n$ be $i.i.d.$ Bernoulli random variable with $\mathbb{E}\epsilon_1=\pi$. The Hoeffding inequality states that, 
for any $0<\delta\leq\frac{1}{2}$,
\begin{equation*}
 \mathbb{P}\Big(\sum\limits_{j=1}^n\epsilon_j\geq n(1+\delta)\pi\Big)\leq e^{-\frac{n\delta^2\pi}{4}}
\end{equation*}
and
\begin{equation*}
 \mathbb{P}\Big(\sum\limits_{j=1}^n\epsilon_j\leq n(1-\delta)\pi\Big)\leq e^{-\frac{n\delta^2\pi}{4}}.
\end{equation*}
Lemma~\ref{ncon} follows immediately by applying Hoeffding inequality.
\end{proof}

\begin{proof}{\bf of Proposition~\ref{prop1}}
 By the definition of $(\hat{\beta}_1,\ldots,\hat{\beta}_{K'})$, we have
 \begin{equation*}
 \begin{split}
  \frac{1}{2}\sum\limits_{k=1}^{K'}&\hat{\beta}_k'S\hat{\beta}_k-\sum\limits_{k=1}^{K'}\hat{\delta}_k'\hat{\beta}_k+\sum\limits_{j=1}^p\|\hat{\beta}^j\|\\
 \leq&  \frac{1}{2}\sum\limits_{k=1}^{K'}\beta_k'S\beta_k-\sum\limits_{k=1}^{K'}\hat{\delta}_k'\beta_k+\sum\limits_{j=1}^p\|\beta^j\|.
 \end{split}
 \end{equation*}
 Denote $\Phi_{\beta}=[\beta_1,\beta_2,\ldots,\beta_{K'}]\in\mathbb{R}^{p\times K'}$ by arranging $\beta_j, j=1,\ldots,K'$ as columns. Simple algebras will lead to
 \begin{equation}
 \label{prop1ineq1}
 \begin{split}
  \sum\limits_{j=1}^p\lambda_j&\|\hat{\beta}^j\|\leq -\frac{1}{2}\sum\limits_{k=1}^{K'}(\hat{\beta}_k-\beta_k)'S(\hat{\beta}_k-\beta_k)\\
  +&\sum\limits_{j=1}^p\lambda_j\|\beta^j\|+\sum\limits_{j=1}^p(e_j'S\Phi_{\beta}-(\hat{\delta}^j)')(\hat{\beta}^j-\beta^j)\\
  \leq&\sum\limits_{j=1}^p\lambda_j\|\beta^j\|+\sum\limits_{j=1}^p\|e_j'S\Phi_{\beta}-(\hat{\delta}^j)'\|\|\hat{\beta}^j-\beta^j\|
  \end{split}
 \end{equation}
Then by Lemma~\ref{devlem1}, for any $1\leq k\leq K'$, we have, conditioned on $\mathcal{A}$, with probability at least $1-\frac{2}{\sqrt{\pi}t}e^{-t^2}-e^{-t}$,
\begin{equation*}
 |e_j'S\beta_k-\hat{\delta}_k^j|\leq C_0\sqrt{\frac{\bar{\pi}(\Delta_k\vee 1)\Sigma_{jj}t}{N-K}}.
\end{equation*}
Therefore, with probability at least $\mathbb{P}(\mathcal{A})-\frac{2Kp}{\sqrt{\pi}t}e^{-t^2}-Kpe^{-t}$, we have for all $j=1,\ldots,p$,
\begin{equation*}
 \|e_j'S\Phi_{\beta}-(\hat{\delta}^j)'\|\leq C_0\sqrt{\frac{\bar{\pi}(\Delta\vee K)\Sigma_{jj}t}{N-K}}.
\end{equation*}
The proof is completed when we plug it into (\ref{prop1ineq1}). When $\lambda_j, j=1,\ldots,p$ are chosen as (\ref{prop1lambda}), we have
\begin{equation*}
 2\sum\limits_{j=1}^p\|\hat{\beta}^j\|\leq 2\sum\limits_{j=1}^p\|\beta^j\|+\sum\limits_{j=1}^p\|\hat{\beta}^j-\beta^j\|.
\end{equation*}
By the fact $\beta^j=0$ for any $j\notin T$, we get
\begin{equation*}
 2\sum\limits_{j\in T}\|\hat{\beta}^j\|+2\sum\limits_{j\notin T}\|\hat{\beta}^j\|\leq 2\sum\limits_{j\in T}\|\beta^j\|+\sum\limits_{j\in T}\|\hat{\beta}^j-\beta^j\|+\sum\limits_{j\notin T}\|\hat{\beta}^j\|.
\end{equation*}
Then we get $\sum_{j\notin T}\|\hat{\beta}^j\|\leq 3\sum_{j\in T}\|\hat{\beta}^j-\beta^j\|$.
\end{proof}

\begin{proof}{\bf of Proposition~\ref{prop2}}
From the proof of Proposition~\ref{prop1}, we have on the event $\mathcal{B}_t$,
\begin{equation*}
\begin{split}
 \frac{1}{2}\sum\limits_{k=1}^{K'}(\hat{\beta}_k-\beta_k)'&S(\hat{\beta}_k-\beta_k)+\lambda\sum\limits_{j=1}^p\|\hat{\beta}^j\|\\
 \leq&\lambda\sum\limits_{j=1}^p\|\beta^j\|+\frac{\lambda}{2}\sum\limits_{j=1}^p\|\hat{\beta}^j-\beta^j\|.
 \end{split}
\end{equation*}
Together with Proposition~\ref{prop1}, we get $\sum_{k=1}^{K'}(\hat{\beta}_k-\beta_k)'S(\hat{\beta}_k-\beta_k)\leq 3\lambda\sum_{j=1}^p\|\hat{\beta}^j-\beta^j\|\leq 12\lambda\sum_{j\in T}\|\hat{\beta}^j-\beta^j\|$.
Furthermore, we have
\begin{equation*}
 \begin{split}
  \sum\limits_{k=1}^{K'}(\hat{\beta}_k-\beta_k)'&S(\hat{\beta}_k-\beta_k)=\sum\limits_{k=1}^{K'}\sum\limits_{i,j=1}^p(\hat{\beta}_k^i-\beta_k^i)S_{ij}(\hat{\beta}_k^j-\beta_k^j)\\
  =&\sum\limits_{k=1}^{K'}\Big[\sum\limits_{i}^p(\hat{\beta}_k^i-\beta_k^i)^2S_{ii}+\sum\limits_{i\neq j}^p(\hat{\beta}_k^i-\beta_k^i)S_{ij}(\hat{\beta}_k^j-\beta_k^j)\Big]\\
  \geq&S^+_{\min}\sum\limits_{j=1}^p\|\hat{\beta}^k-\beta^k\|^2+\sum\limits_{i\neq j}^p\sum\limits_{k=1}^{K'}(\hat{\beta}_k^i-\beta_k^i)S_{ij}(\hat{\beta}_k^j-\beta_k^j)\\
    \geq&S^+_{\min}\sum\limits_{j=1}^p\|\hat{\beta}^k-\beta^k\|^2-S^-_{\max}\sum\limits_{i\neq j}^p\|\hat{\beta}^i-\beta^i\|\|\hat{\beta}^j-\beta^j\|.
 \end{split}
\end{equation*}
Therefore, on the event $\mathcal{D}$, we have 
\begin{equation*}
\begin{split}
\sum_{k=1}^{K'}(\hat{\beta}_k-\beta_k)'&S(\hat{\beta}_k-\beta_k)\geq \frac{\Sigma^+_{\min}}{2}\sum_{j=1}^p\|\hat{\beta}^j-\beta^j\|^2-2\Sigma^-_{\max}(\sum\limits_{j=1}^p\|\hat{\beta}^j-\beta^j\|)^2\\
\geq& \frac{\Sigma^+_{\min}}{2}\sum_{j=1}^p\|\hat{\beta}^j-\beta^j\|^2-32\Sigma^-_{\max}(\sum\limits_{j\in T}\|\hat{\beta}^j-\beta^j\|)^2\\
\geq& \frac{\Sigma^+_{\min}}{2}\sum_{j=1}^p\|\hat{\beta}^j-\beta^j\|^2-32s\Sigma^-_{\max}\sum\limits_{j\in T}\|\hat{\beta}^j-\beta^j\|^2
\end{split}
\end{equation*}
Then (\ref{prop2ineq1}) is an immediate result. In the case that $\frac{\Sigma^+_{\min}}{2}-32s\Sigma^-_{\max}\geq \frac{\Sigma^+_{\min}}{4}$, (\ref{prop2ineq1})
indicates that
\begin{equation*}
 \frac{\Sigma^+_{\min}}{4}\sum\limits_{j\in T}\|\hat{\beta}^j-\beta^j\|^2\leq 12\lambda\sum\limits_{j\in T}\|\hat{\beta}^j-\beta^j\|\leq 12\lambda\sqrt{s\sum\limits_{j\in T}\|\hat{\beta}^j-\beta^j\|^2},
\end{equation*}
which leads to (\ref{prop2ineq2}). Similarly we can show (\ref{prop2ineq3}).
\end{proof}

\begin{proof}{\bf of Proposition~\ref{prop3}}
 By applying KKT condition to (\ref{est}), we get for any $1\leq j\leq p$,
 \begin{equation*}
  \|\Phi_{\hat{\beta}}'Se_j-\hat{\delta}^j\|\leq \lambda,
 \end{equation*}
where $\Phi_{\hat{\beta}}$ is defined similarly in the proof of Proposition~\ref{prop1}. Therefore, we get
\begin{equation*}
 \|e_j'S\Phi_{\hat{\beta}-\beta}\|\leq \|\hat{\delta}^j-\delta^j\|+\lambda\leq C_0\sqrt{\frac{\bar{\pi}K\Sigma_{jj}t}{N-K}}+\lambda\leq2\lambda,
\end{equation*}
where the control on $\|\hat{\delta}^j-\delta^j\|$ is based on the event $\mathcal{B}_t$ and Lemma~\ref{devlem1}. By rewriting $e_j'S(\hat{\beta}_k-\beta_k)$ as
$S_{jj}(\hat{\beta}_k^j-\beta_k^j)+\sum_{i\neq j}S_{ji}(\hat{\beta}_k^i-\beta_k^i)=:a_k+b_k$ for $1\leq k\leq K'$, we have
\begin{equation*}
 \|(a_1,a_2,\ldots,a_{K'})\|\leq \|e_j'S\Phi_{\hat{\beta}-\beta}\|+\|(b_1,b_2,\ldots,b_{K'})\|.
\end{equation*}
If $S^+_{\min}\geq \frac{1}{2}\Sigma^+_{\min}$ and $S^-_{\max}\leq 2\Sigma^-_{\min}$, we get
\begin{equation*}
 \begin{split}
  \frac{\Sigma^+_{\min}}{2}\|&\hat{\beta}^j-\beta^j\|\leq 2\lambda+\|(b_1,b_2,\ldots,b_{K'})\|\\
  \leq&2\lambda+S^-_{\max}\sqrt{\sum\limits_{k=1}^{K'}(\sum\limits_{i\neq j}|\hat{\beta}_k^i-\beta_k^i|)^2}\\
  \leq&2\lambda+2\Sigma^-_{\max}\sqrt{\sum\limits_{k=1}^{K'}\sum\limits_{i\neq i'}|\hat{\beta}^i_k-\beta^i_k||\hat{\beta}^{i'}_k-\beta^{i'}_k|}\\
  \leq&2\lambda+2\Sigma^-_{\max}\sqrt{\sum\limits_{i\neq i'}\|\hat{\beta}^i-\beta^i\|\|\hat{\beta}^{i'}-\beta^{i'}\|}\leq2\lambda+2\Sigma^-_{\max}\sum\limits_{j=1}^p\|\hat{\beta}^j-\beta^j\|\\
  \leq&2\lambda+8\Sigma^-_{\max}\sum\limits_{j\in T}\|\hat{\beta}^k-\beta^k\|\leq2\lambda+8\Sigma^-_{\max}\sqrt{s\sum\limits_{j\in T}\|\hat{\beta}^k-\beta^k\|^2}\\
  \leq&2\lambda+\frac{8\sqrt{C_1}\lambda s\Sigma^-_{\max}}{\Sigma^+_{\min}}\leq (2+\frac{\sqrt{C_1}}{16})\lambda,
 \end{split}
\end{equation*}
where the last inequality is due to Proposition~\ref{prop2}. Then we get $\frac{\Sigma^+_{\min}}{2}\|\hat{\beta}^j-\beta^j\|\leq C_2\lambda$.
\end{proof}

\begin{proof}{\bf of Theorem~\ref{mainthm2}}\label{proofintromainthm2}
For any $j\in T$, by $\|\bar{\beta}^j-\beta^j\|=\sum_{j=1}^{K'}(\bar{\beta}_k^j-\beta_k^j)^2$, 
we have $\mathbb{E}\|\bar{\beta}^j-\beta^j\|=\sum_{k=1}^{K'}\mathbb{E}(\bar{\beta}_k^j-\beta_k^j)^2$. For any $1\leq k\leq K'$ such that
$j\in T_k$, we compute $\mathbb{E}(\bar{\beta}_k^j-\beta_k^j)^2$.
Define
\begin{equation*}
 \alpha_j:=\left(e_j'\hat{\Sigma}_{T_k,T_k}^{-1}\hat{\delta}_{k}^{T_k}-e_j'\Sigma_{T_k,T_k}^{-1}\delta_{k}^{T_k}\right)^2=\left((N-K)e_j'A\hat{\delta}_{k}^{T_k}-e_j'\Sigma_{T_k,T_k}^{-1}\delta_{k}^{T_k}\right)^2,
\end{equation*}
where $A\sim\mathcal{W}_{s_k}^{-1}(\Sigma_{T_k,T_k}^{-1},N-K)$, an inverted Wishart distribution.
Therefore, conditioned on $n_1,\ldots,n_K$,
\begin{equation*}
\begin{split}
\mathbb{E}[\alpha_j|n_1,\ldots,n_K]=&\mathbb{E}\Big((N-K)^2(e_j'A\hat{\delta}_k^{T_k})^2\\
-&2(N-K)(e_j'\Sigma_{T_k,T_k}^{-1}\delta_k^{T_k})(e_j'A\hat{\delta}_k^{T_k})+(e_j'\Sigma_{T_k,T_k}^{-1}\delta_k^{T_k})^2\Big|n_1,\ldots,n_K\Big)
\end{split}
\end{equation*}
By property of inverted Wishart matrix, \cite[Theorem 5.2.2]{press2012applied}, $\mathbb{E}A=\frac{\Sigma_{T_k,T_k}^{-1}}{N-K-s_k-1}$. Meanwhile, $\mathbb{E}\hat{\delta}_k^{T_k}=\delta_k^{T_k}$ and $A, \hat{\delta}_k^{T_k}$ are independent.
Then, we get,
\begin{equation*}
 \mathbb{E}[\alpha_j|n_1,\ldots,n_K]=(N-K)^2\mathbb{E}[(e_j'A\hat{\delta}_k^{T_k})^2|n_1,\ldots,n_K]-\frac{N-K+s_k+1}{N-K-s_k-1}(e_j'\Sigma_{T_k,T_k}^{-1}\delta_k^{T_k}).
\end{equation*}
Consider $\mathbb{E}Ae_je_j'A=\mathbb{E}A_{j:}'A_{j:}=\textrm{Cov}(A_{j:}')+\mathbb{E}A_{j:}'\mathbb{E}A_{j:}=\textrm{Cov}(A_{j:}')+\frac{1}{(N-K-s_k-1)^2}\Sigma_{T_k,T_k}^{-1}e_je_j'\Sigma_{T_k,T_k}^{-1}$.
Since $\hat{\delta}_k^{T_k}\sim\mathcal{N}(\delta_k^{T_k},\frac{n_1+n_{k+1}}{n_1n_{k+1}}\Sigma_{T_k,T_k})$, we have,
\begin{equation*}
\begin{split}
 \mathbb{E}[\alpha_j|n_1,\ldots,n_K]=&(N-K)^2\mathbb{E}\Big[\big<\textrm{Cov}(A_{j:}')\hat{\delta}_k^{T_k},\hat{\delta}_k^{T_k}\big>|n_1,n_{k+1}\Big]\\
 +&\frac{(N-K)^2}{(N-K-s_k-1)^2}\mathbb{E}[(e_j'\Sigma_{T_k,T_k}^{-1}\hat{\delta}_k^{T_k})^2|n_1,n_{k+1}]-\frac{N-K+s_k+1}{N-K-s_k-1}(e_j'\Sigma_{T_k,T_k}^{-1}\delta_k^{T_k})^2\\
=&(N-K)^2\mathbb{E}\Big[\big<\textrm{Cov}(A_{j:}')\hat{\delta}_k^{T_k},\hat{\delta}_k^{T_k}\big>|n_1,n_{k+1}\Big]+\frac{(s_k+1)^2}{(N-K-s_k-1)^2}(e_j'\Sigma_{T_k,T_k}^{-1}\delta_k^{T_k})^2\\
+&\frac{(N-K)^2}{(N-K-s_k-1)^2}\frac{n_1+n_{k+1}}{n_1n_{k+1}}\big(\Sigma_{T_k,T_k}^{-1}\big)_{j,j}
\end{split}
\end{equation*}
The first term can be calculated as,
\begin{equation*}
 \begin{split}
  \mathbb{E}\Big[\big<\textrm{Cov}(A_{j:}')\hat{\delta}_k^{T_k},\hat{\delta}_k^{T_k}\big>|n_1,n_{k+1}\Big]=&\big<\textrm{Cov}(A_{j:}')\delta_k^{T_k},\delta_k^{T_k}\big>+\Big<\frac{n_1+n_{k+1}}{n_1n_{k+1}}\Sigma_{T_k,T_k},\textrm{Cov}(A_{j:}')\Big>.
 \end{split}
\end{equation*}
Since, \cite[Theorem 5.2.2]{press2012applied}
$$
\textrm{Cov}(A_{ij},A_{ik})=\frac{(N-K-s_k+1)(\Sigma_{T_k,T_k}^{-1})_{ij}(\Sigma_{T_k,T_k}^{-1})_{ik}+(N-K-s_k-1)(\Sigma_{T_k,T_k}^{-1})_{ii}(\Sigma_{T_k,T_k}^{-1})_{jk}}{(N-K-s_k)(N-K-s_k-1)^2(N-K-s_k-3)},
$$
we get that
\begin{equation*}
 \begin{split}
  \big<\textrm{Cov}(A_{j:}')\delta_k^{T_k},\delta_k^{T_k}\big>=&\frac{(N-K-s_k+1)}{(N-K-s_k)(N-K-s_k-1)^2(N-K-s_k-3)}(e_j'\Sigma_{T_k,T_k}^{-1}\delta_k^{T_k})^2\\
+&\frac{(\Sigma_{T_k,T_k}^{-1})_{j,j}\Delta_k}{(N-K-s_k)(N-K-s_k-1)(N-K-s_k-3)}
 \end{split}
\end{equation*}
and
\begin{equation*}
 \begin{split}
  \big<\frac{n_1+n_{k+1}}{n_1n_{k+1}}\Sigma_{T_k,T_k},\textrm{Cov}(A_{j:}')\big>=&\frac{(n_1+n_{k+1})(N-K-s_k+1)\big<\Sigma_{T_k,T_k}^{-1}e_je_j'\Sigma_{T_k,T_k}^{-1},\Sigma_{T_k,T_k}\big>}{n_1n_{k+1}(N-K-s_k)(N-K-s_k-1)^2(N-K-s_k-3)}\\
+&\frac{(n_1+n_{k+1})s_k(\Sigma_{T_k,T_k}^{-1})_{j,j}}{n_1n_{k+1}(N-K-s_k)(N-K-s_k-1)(N-K-s_k-3)}.
 \end{split}
\end{equation*}
Putting together all the results above, we have that
\begin{equation*}
 \mathbb{E}\alpha_j\approx \frac{s_k^2+N-K}{(N-K)^2}(\beta_k^j)^2+\frac{2(\pi_1+\pi_{k+1})}{(N-K)\pi_1\pi_{k+1}}(\Sigma_{T_k,T_k}^{-1})_{j,j}+\frac{(\Sigma_{T_k,T_k}^{-1})_{j,j}}{N-K}\Delta_k.
\end{equation*}
The remaining of the proof is straightforward.
\end{proof}
\vskip 0.2in 
%----------------------------------------------------------------------------------------
%	REFERENCE LIST
%----------------------------------------------------------------------------------------
\bibliographystyle{abbrv}
\bibliography{refer}

\end{document}